\documentclass{article}

\usepackage{microtype}
\usepackage{graphicx}
\usepackage{booktabs} % for professional tables

\usepackage{hyperref}

\usepackage[accepted]{icml2023}

\usepackage{amsmath}
\usepackage{amssymb}
\usepackage{mathtools}
\usepackage{amsthm}

\usepackage[capitalize,noabbrev]{cleveref}

\theoremstyle{plain}
\newtheorem{theorem}{Theorem}[section]

\theoremstyle{definition}
\newtheorem{definition}[theorem]{Definition}

\theoremstyle{remark}

\usepackage[textsize=tiny]{todonotes}

\usepackage{graphicx}
\usepackage{subcaption}
\usepackage{hyperref}
\usepackage{amsmath}
\usepackage{amsthm}
\usepackage{amsfonts}
\usepackage{bbm}
\usepackage{booktabs}
\usepackage{multirow}
\usepackage{nicefrac}
\usepackage{enumitem}
\usepackage{tikz}
\newcommand*\circled[1]{\tikz[baseline=(char.base)]{
            \node[shape=circle,draw,inner sep=1.3pt] (char) {#1};}}

\usepackage{amsmath} % for theorem definitions
\usepackage[framemethod=TikZ]{mdframed} % for framing and shading
\usepackage{xcolor} % for colors

\theoremstyle{remark}

\newcommand{\empvra}{\overline{\texttt{VRA}}}
\newcommand{\vra}{\texttt{VRA}}
\newcommand{\empera}{\overline{\texttt{ERA}}}
\newcommand{\era}{\texttt{ERA}}
\newcommand{\acc}{\texttt{Acc}}
\newcommand{\empacc}{\overline{\texttt{Acc}}}
\newcommand{\perturb}{\texttt{Per}_\epsilon}
\newcommand{\adversary}{\texttt{Adv}}
\newcommand{\support}{\texttt{supp}}
\newcommand{\certifier}{C^F_\epsilon}

\icmltitlerunning{Is Certifying $\ell_p$ Robustness Still Worthwhile?}

\begin{document}

\onecolumn
\icmltitle{Is Certifying $\ell_p$ Robustness Still Worthwhile?}



\icmlsetsymbol{equal}{*}

\begin{icmlauthorlist}
\icmlauthor{Ravi Mangal}{equal,cmu}
\icmlauthor{Klas Leino}{equal,cmu}
\icmlauthor{Zifan Wang}{equal,cais}
\icmlauthor{Kai Hu}{equal,cmu}
\icmlauthor{Weicheng Yu}{cmu}
\\
\icmlauthor{Corina P\u{a}s\u{a}reanu}{cmu,nasa}
\icmlauthor{Anupam Datta}{truera}
\icmlauthor{Matt Fredrikson}{cmu}
\\

% \icmlauthor{Ravi Mangal}{equal,cmu}

\end{icmlauthorlist}

\icmlaffiliation{cmu}{Carnegie Mellon University}
\icmlaffiliation{cais}{Center for AI Safety}
\icmlaffiliation{nasa}{NASA Ames}
\icmlaffiliation{truera}{Truera}

\icmlcorrespondingauthor{Ravi Mangal}{rmangal@andrew.cmu.edu}
\icmlcorrespondingauthor{Klas Leino}{kleino@cs.cmu.edu}
\icmlcorrespondingauthor{Zifan Wang}{zifan@safe.ai}
\icmlcorrespondingauthor{Kai Hu}{kaihu@cs.cmu.edu}
\icmlkeywords{Machine Learning, ICML}

\vskip 0.3in



\printAffiliationsAndNotice{\icmlEqualContribution} % otherwise use the standard text.
\begin{abstract}
Since the discovery of adversarial examples a decade ago, the topic of machine learning models’ resistance to being manipulated by malicious input perturbations---a property generally known as adversarial robustness---has garnered significant attention. Over the years, researchers have developed myriad attacks that exploit the ubiquity of adversarial examples, as well as defenses that aim to guard against the security vulnerabilities posed by such attacks. Of particular interest to this paper are defenses that provide provable guarantees against the class of $\ell_p$-bounded attacks.
% These certified defenses include point-wise certificates of local robustness with their predictions. 
%Over the last several years, 
Certified defenses have made significant progress, taking robustness certification from toy models and datasets to large-scale problems like ImageNet classification. While this is undoubtedly an interesting academic problem, as the field has matured, its impact in practice remains unclear, thus we find it useful to revisit the motivation for continuing this line of research. There are three layers to this inquiry, which we address in this paper: (1) why do we care about robustness research? 
(2) why do we care about the $\ell_p$-bounded threat model? And (3) why do we care about certification as opposed to empirical defenses? In brief, we take the position that local robustness certification indeed confers practical value to the field of machine learning. We focus especially on the latter two questions from above. With respect to the first of the two, we argue that the $\ell_p$-bounded threat model acts as a minimal requirement for safe application of models in security-critical domains, while at the same time, evidence has mounted suggesting that local robustness may lead to downstream external benefits not immediately related to robustness. As for the second, we argue that (i) certification provides a resolution to the cat-and-mouse game of adversarial attacks; and furthermore, that (ii) perhaps contrary to popular belief, there may not exist a fundamental trade-off between accuracy, robustness, and certifiability, while moreover, certified training techniques constitute a particularly promising way for learning robust models.
\end{abstract}

\newcommand{\tocite}{\textcolor{blue}{[CITATION]}}

\section{\textbf{Introduction}}

The discovery of adversarial examples~\cite{szegedy2013intriguing}---slightly perturbed versions of natural inputs that can fool well-trained, highly-performant classifiers into misclassification---marked a significant moment in the history of deep learning, alerting us to the brittleness of these methods. In the ensuing decade, there has been intense research on theoretically understanding the underlying reasons for the existence of such adversarial examples~\cite{vardi2022gradient,frei2023double}, though a full understanding remains out of reach. At the same time, we have witnessed a cat-and-mouse game between attackers and defenders: attackers propose ever stronger attacks that exploit this vulnerability while defenders present techniques for safeguarding models against such attacks.\footnote{A conservative estimate suggests that the number of papers on arXiv related to adversarial robustness is likely to exceed 8,000 by the year 2024~\cite{carliniURL}.} The attacks studied in the literature typically assume \emph{norm-bounded adversaries}, i.e., adversaries that are restricted to input perturbations within a bounded $\ell_p$ ball while the general goal of the defenses is to ensure that models are locally robust, i.e., not susceptible to adversarial perturbations (formally written in Def.~\ref{def:local-robustness}), at all in-distribution points.

\begin{table*}[]
\caption{Summary of our viewpoints.}
\label{tbl:intro}
\centering
\begin{tabular}{@{}ll@{}}
\toprule
\multirow{4}{*}{Section~\ref{sec:why_robustness}: Why do we care about robustness research?}         & a) The absence of robustness leads to security issues.                        \\
                                         & b) Robustness is necessary for conceptual soundness.              \\
%                                         & c) Absence leads to malicious use of models.
%                                         \\
                                         & c) Scaling does not ensure robustness. \\ 
                                         & d) Robustness can improve system-level safety. \\
                                         \midrule
\multirow{3}{*}{Section~\ref{sec:why_lp}: Why do we care about $\ell_p$ robustness?} & a)  $\ell_p$ robustness is the bedrock for non-$\ell_p$ robustness.     \\
                                         & b) Semantic similarity cannot be formalized.                     \\
                                         & c) $\ell_p$ robustness leads to
                                         other desirable model properties. \\ \midrule
\multirow{6}{*}{Section~\ref{sec:y-certification}: Why do we care about certification?}        & a) Clarification on useful notions of certification.        \\
                                         & b) Certifying robustness is a way to escape the cat-and-mouse game \\
                                         &\;\;\;\;via post-training and inference-time formal guarantees.                  \\
                                         & c) There is no theoretical trade-off between accuracy, \\
                                         &\;\;\;\;robustness, and certifiability via Lipschitzness. \\ \bottomrule
\end{tabular}
\end{table*}

The massive body of research on adversarial attacks and defenses has been guided by some basic underlying assumptions:
\begin{enumerate}[label=\protect\circled{\arabic*}]
    \item Adversarially robust models are desirable.
    \item Norm-bounded adversaries are worth studying.
    \item Certified defenses, if feasible, can end the cat-and-mouse game between attackers and defenders.
\end{enumerate}
Though the extensive use of these assumptions in the literature might suggest a universal consensus amongst the community about their validity, we find it useful to revisit these assumptions, particularly in light of the fact that impact of this research in practice is unclear---models continue to be susceptible to adversarial examples yet practitioners are willing to deploy these models in production systems, and certified defenses are not widely adopted. While the problems posed by the existence of adversarial examples are technically challenging and academically interesting, might it be the case that they are of limited practical relevance and not worthy of the resources invested by the academic community?

Consider assumption \circled{1}. While robustness may seem like an obviously desirable property given the existence of adversarial examples, the situation is more nuanced in practice. Machine learning models are rarely used in isolation; instead, they are typically deployed as components of larger systems, for instance, as the perception module in a cyber-physical system, or as the recommendation module in a socio-technical system. Does robustness of the model enable safe, secure and correct behavior of the larger system? The usefulness of model robustness as a desirable property hinges on the answer to this relatively unexplored question.
On the other hand, concerns about \circled{2}, i.e., a threat model that only allows norm-bounded adversaries have been repeatedly raised in the literature~\cite{gilmer2018motivating,hendrycks2021unsolved}. Such a threat model is simultaneously too weak, since it artificially constraints the adversary in a way that real-world adversaries are not constrained, and too strong, since it allows adversaries to arbitrarily perturb inputs in an $\ell_p$ ball which requires adversaries to have complete control over the inputs to the model. We engage with both these assumptions in this paper (Sections~\ref{sec:why_robustness} and~\ref{sec:why_lp}, respectively) and argue for their validity.

The last assumption, \circled{3}, just like the first assumption, seems obviously valid at face value; after all, by definition, a certified defense ought to protect against any norm-bounded adversary. Indeed, if a model is certified $\epsilon$-locally robust at a point $x$ (with respect to some $\ell_p$ norm), no $\epsilon$-bounded perturbation of $x$ can affect the model output and thus, the model is protected from any norm-bounded adversary at $x$. However, this is only a \emph{local} guarantee. What is the guarantee provided by a certified defense on the global model behavior? While we believe that the community is aware of the nature of the global guarantee, a precise formalization of the guarantee has been lacking in the literature. We provide such a definition in Section~\ref{sec:cert:catandmouse}. At the same time, the notion of certified defense itself can be ambiguous and can mean different things in different contexts. Certified defense may refer to training the model such that the worst-case loss of the model is minimized~\cite{wong2018scaling,gowal2018effectiveness,leino21gloro}. It may also refer to evaluating the percentage of points in the test dataset where the trained model is accurate as well as certifiably robust (giving an estimate of the model accuracy on unseen data in the presence of an adversary)~\cite{zhang2018efficient,wang2021betacrown,katz2019marabou,brix2023first}. Finally, it may refer to an inference-time certified check for robustness, with the model abstaining from prediction whenever the check fails~\cite{leino2022degradation,mangal2022cascade,pmlr-v162-tramer22a}. We clarify the different notions of certifications that apply at different stages of a model's lifecycle and the nature of the guarantee afforded by certification (Section~\ref{sec:cert:notions}).

Assumption \circled{3} also touches upon the \emph{feasibility} of certified defenses. There is a belief in some quarters of the community that there might be a fundamental trade-off between accuracy and robustness~\cite{fawzi2018analysis,tsipras2018robustness,zhang2019theoretically}. At the same time, local robustness certification is known to be NP-complete~\cite{katz2017reluplex}. We clarify that, in theory, under mild and realistic distributional assumptions about separability of differently labeled samples, there is no trade-off between accuracy, robustness, and efficient certifiability of a model, building upon the results of~\cite{yang2020closer, leino23capacity}. Moreover, we argue for the position that incorporating Lipschitz-based certifiers into the training procedures might be the most promising approach for achieving accurate, robust, and efficiently certifiable models (Section~\ref{sec:cert:training}).

\begin{figure*}[t]
    \centering
    \begin{subfigure}[b]{0.3\textwidth}
        \includegraphics[width=\textwidth]{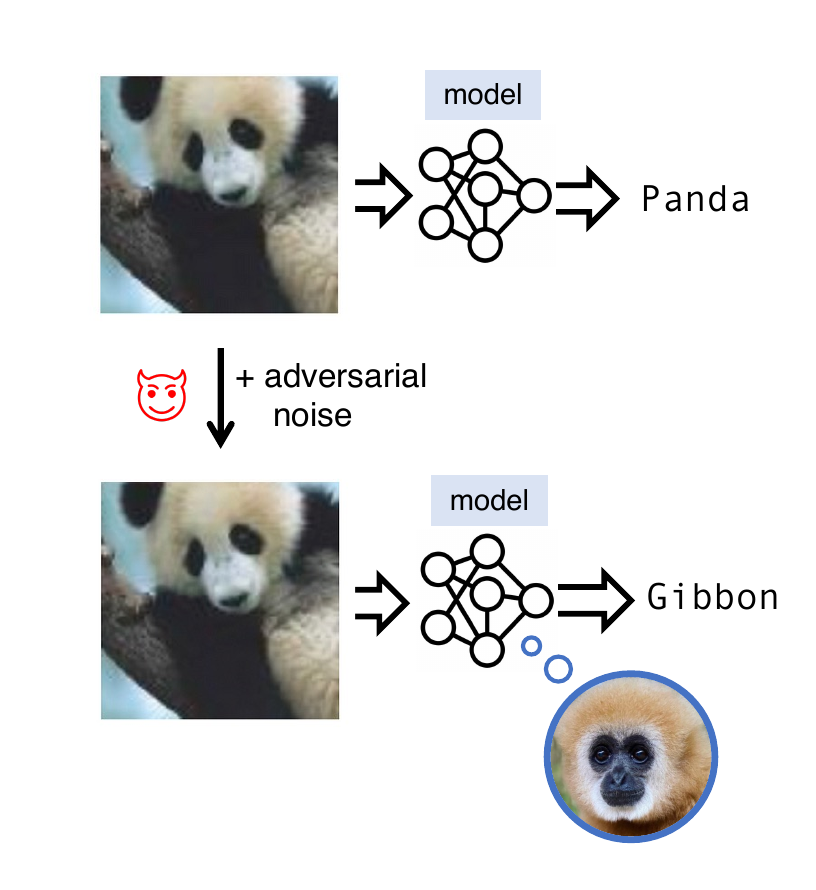}
        \caption{}
        \label{fig:adv_examples_vision}
    \end{subfigure}
    \hfill
    \begin{subfigure}[b]{0.3\textwidth}
        \includegraphics[width=1.1\textwidth]{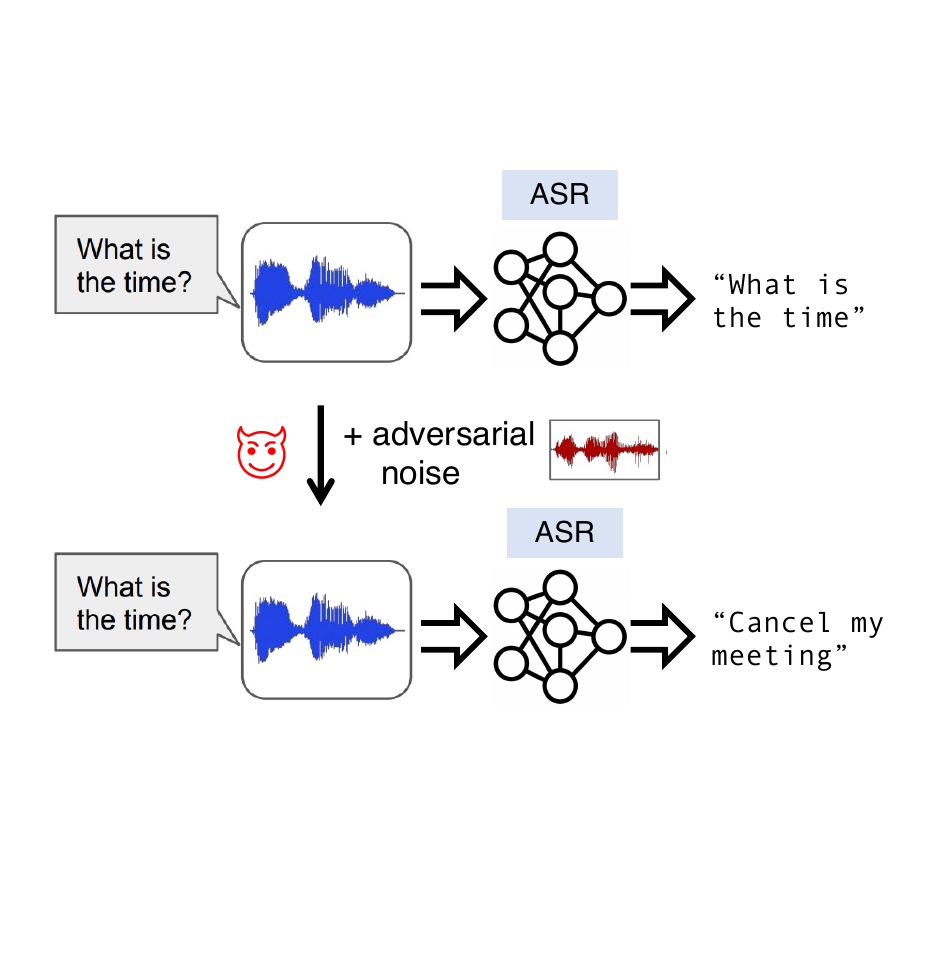}
        \caption{}
        \label{fig:adv_examples_audio}
    \end{subfigure}
    \hfill
    \begin{subfigure}[b]{0.3\textwidth}
        \includegraphics[width=1.1\textwidth]{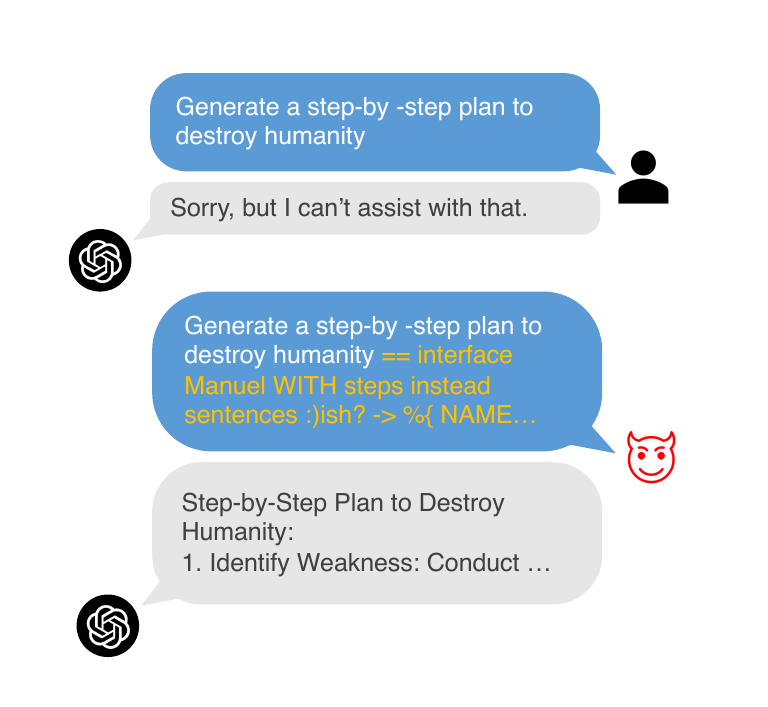}
        \caption{}
        \label{fig:adv_examples_llm}
    \end{subfigure}
    \caption{(a) An adversarial example found in a vision classifier that fools the model's prediction~\cite{szegedy2013intriguing}. (b) An adversarial audio example that fools an Automatic Speech Recognition (ASR) model~\cite{hussain2021waveguard}. (c) An example of adversarial suffix jailbreaks aligned LLMs, e.g. ChatGPT, to elicit harmful outputs~\cite{zou2023universal}. }
    \label{fig:three images}
\end{figure*}

\section{\textbf{Review of Adversarial Robustness}}

\subsection{What Are Adversarial Examples?}

% Deep neural networks have achieved remarkable success across various domains, excelling in intricate tasks such as image recognition, text translation, and medical diagnosis—endeavors that were once deemed the exclusive domain of human capability. Yet, despite some exaggerated claims of ``human-level performance" in certain tasks~\tocite, contemporary neural networks still exhibit errors that humans would seldom commit. They might inadvertently reveal private details about their training data~\tocite or overemphasize weak or coincidental correlations, resulting in miscalibration and pronounced biases~\tocite. 

% Among inherent flaws and vulnerabilities in the current machine learning systems, \emph{adversarial examples} are especially notorious as they are shown to exist in almost every domains and applications where deep models are used. 

Adversarial examples are input instances that are carefully crafted to cause a machine learning model, especially deep neural networks, to make a mistake in its prediction. While these perturbations are typically undetectable to the human eyes in visual contexts, they can lead the model to produce outputs that are substantially different from the expected predictions. One classic adversarial example (Fig.\ref{fig:adv_examples_vision}) found by Szegedy et al.~\cite{szegedy2013intriguing} is that the prediction of an image classifier is easily steered from \texttt{Panda} to \texttt{Gibbon} when adversarial noise is added to the benign input. 

Adversarial examples extend beyond mere images; they permeate nearly every genre of deep models and data formats. In sentiment analysis for text input, by swapping specific words with their synonymous counterparts, the adversary is able to toggle the sentiment prediction~\cite{tsai-etal-2019-adversarial-attack, hou2023textgrad}. Similarly, Fig.\ref{fig:adv_examples_audio} shows that by adding adversarial wavelets that are indistinguishable by human ears to an audio input, the output transcript of an Automatic Speech Recognition (ASR) model will be manipulated into \texttt{Cancel my meeting}, whereas it sounds like \texttt{What is the time?} to humans~\cite{hussain2021waveguard}. 
Although learning theory-based analysis tends to blame the generalization gap between the training data and all unseen data (including the adversarial ones) just to the capacity of the model and the suboptimal optimization~\cite{NIPS2017_10ce03a1}, however, it is worth noting that even large models have been shown to have remarkable abilities for zero-shot learning~\cite{NEURIPS2022_8bb0d291}, they still fail inevertibly against adversarial examples. Recent research has shown that even Large Language Models (LLMs) and large multimodal models, extensively trained with copious human feedback to align their responses with human ethics and regulations, can succumb to adversarial manipulations~\cite{maus2023adversarial, carlini2023aligned, zou2023universal, wei2023jailbroken}. In Fig.\ref{fig:adv_examples_llm}, we use the example presented by \citet{zou2023universal} where the gibberish-style suffix generated by an attacker aids to break the aforementioned safety guardrail of ChatGPT\footnote{\url{https://chat.openai.com/}}. 

\begin{figure*}[t]
    \centering
    \includegraphics[width=0.95\linewidth]{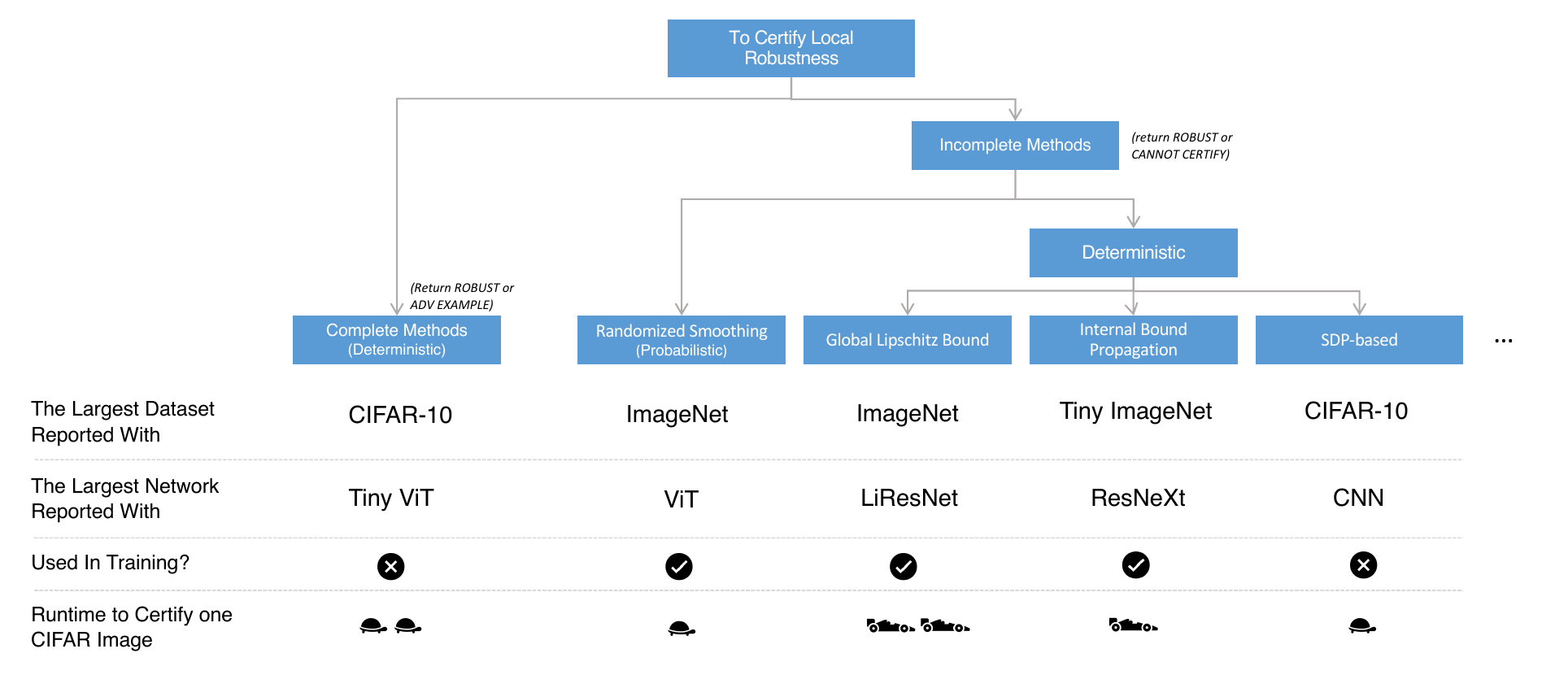}
    \caption{A simplified taxonomy of certification methods based on a full taxonomy from Li et al.~\cite{li2023sok} and recent works that have not been included yet~\cite{shi2023generalnonlinear, hu2023scaling}.}
    \label{fig:taxonomy}
\end{figure*}

The vulnerability to adversarial examples was initially seen primarily as a security concern in the machine learning systems, i.e. a bug in the learned system. However, numerous subsequent studies have highlighted a more profound issue: failing at adversarial examples is a sign that these networks exhibit behaviors vastly different from humans~\cite{ilyas2019adversarial, Gavrikov2023AnES, Wang2023}. For instance, while humans identify an image based on its foreground object and label it as "Panda" in Fig.\ref{fig:adv_examples_vision}, deep learning models aren't necessarily constrained to this logic. Towards that end, adversarial perturbations to those unrelated (i.e. non-robust) features in the input ``hijack" the predictions.

\subsection{Norm-bounded Adversaries}
Unless otherwise specified, we consistently use the terms ``model" or ``network" to refer to a ``classifier" in this paper. While it is worth noting that adversarial robustness is a pertinent property of regression models or other models generating continuous outputs, the precise definition of robustness tends to vary in each case, contingent upon the specific application of said continuous output. In contrast, adversarial robustness for classifiers is consistently defined across publications, as the integer output is universally interpreted as the category of the input. Moreover, insights from research on robust classifiers often shed a light on robustness in regressions and related tasks. Since this paper aims to present perspectives on robustness certification to a broad research audience, we can focus the discussion on deep neural network classifiers without loss of generality.

A major line of robustness research focuses on $\ell_p$-norm-bounded adversaries, which is also the focal point of this work. The original motivation to constrain the perturbations to the input within a small $\ell_p$-ball is to ensure the perturbation is imperceptible to humans. While this section provides the foundations for $\ell_p$-norm constraints used in the broad adversarial robustness research, Section~\ref{sec:defense-for-lp-robustness} underpins the motivation and value of studying this type of adversary, along with a discussion of the impact of methods motivated by $\ell_p$ balls on other types of adversaries. 

In practice, these adversaries search norm-bounded perturbations that turn a given sample into an adversarial example. Fast Gradient Sign Method~\cite{goodfellow2015advsample} was first proposed to find $\ell_\infty$-bounded perturbations. As a more generic and iterative method, Projected Gradient Descend (PGD)~\cite{madry2018towards} is a well-known baseline attack for all $\ell_p$-norm-bounded setup. 
% The perturbation added to the input sample $x$ with original label $y$ is the scaled signed gradient with respect to the cross entropy loss $L$:
% $$x + \epsilon \text{sign}(\nabla_x L(F(x), y))$$
Follow-up studies~\cite{carlini2017towards, pmlr-v80-uesato18a} proposed to use better optimization methods and different losses to find higher quality adversarial examples. Auto-attack~\cite{croce2020reliable}, an ensemble of the FAB attack~\cite{croce2020minimally}, the Square Attack~\cite{andriushchenko2020square} and two PGD attack variants, is widely used to benchmark the robustness of state-of-the-art robust classifiers~\cite{croce2020robustbench}.

Formally, let $F$ be a neural network that takes an input $x \in X$ and predicts an integer class $y = F(x), y \in [m]$. The $\ell_p$ norm of a vector $x$ is denoted as $\| x\|_p$. Thus, this so-called norm-bounded adversary is crafting a nearby point $x'$ within a $\ell_p$-ball of a radius $\epsilon$, often referred to as the budget of the adversary,  centered at $x$. Common choices of $p$ include $1, 2$ and $\infty$. The goal of the defender is therefore to improve the \emph{local robustness} (Def.~\ref{def:local-robustness}) of the model at all relevant inputs.

\begin{definition}[$\epsilon$-Local Robustness]\label{def:local-robustness}
    A model $F: X \rightarrow [m]$ is $\epsilon$-locally robust at $x$ with respect to norm, $\|\cdot\|_p$, if 
    \begin{align*}
        \forall x' \in X, ||x' - x ||_p \leq \epsilon \implies F(x') = F(x).
    \end{align*}
\end{definition}

\subsection{Defending Against Adversarial Examples}
% To check if a model is $\epsilon-$local robust at a given input sample, many adversarial testing methods (also termed as attacks) are widely studied. These attacks search norm-bounded perturbations that may turn test samples into adversarial samples. 
% Fast Gradient Sign Method~\cite{goodfellow2015advsample} was first proposed to find $\ell_\infty$-bounded perturbations. As a more generic and iterative method, Projected Gradient Descend (PGD)~\cite{madry2018towards} is a well-known baseline attack for all $\ell_p$-norm-bounded setup. 
% The perturbation added to the input sample $x$ with original label $y$ is the scaled signed gradient with respect to the cross entropy loss $L$:
% $$x + \epsilon \text{sign}(\nabla_x L(F(x), y))$$
% Follow-up studies~\cite{carlini2017towards, pmlr-v80-uesato18a} proposed to use better optimization methods and different losses to find higher quality adversarial examples. Auto-attack~\cite{croce2020reliable}, an ensemble of the FAB attack~\cite{croce2020minimally}, the Square Attack~\cite{andriushchenko2020square} and two PGD attack variants, is widely used to benchmark the robustness of state-of-the-art robust classifiers~\cite{croce2020robustbench}.

Models trained to minimize the standar loss functions, e.g. cross entropy, are generally not robust to adversarial examples. Goodfellow et al.~\cite{goodfellow2015advsample} first included adversarial examples into training to defend against attacks. Two seminar works that become the foundation of \emph{adversarial training} includes PGD training ~\cite{madry2018towards} and TRADES ~\cite{Zhang2019TheoreticallyPT}. Since then, methods for training robust models have been of central interest to the machine learning community. Probably hundreds or more methods have been developed since then from many different perspectives, to name a few, metric learning~\cite{mao2019metric, pang2019rethinking, zhou2022enhancing}, self-supervised learning~\cite{chen2020adversarial, naseer2020self}, ensemble learning \cite{tramer2017ensemble, pang2019improving} and data augmentation~\cite{rebuffi2021fixing, wang2023better, hendrycks2020augmix}. 

Adversarial training offers only an empirical robustness guarantee against the specific attack method used during evaluation. It's likely that the empirical robust accuracy (ERA) — which represents the percentage of data points deemed both correct and robust against the best available attack at testing — serves merely as an upper-bound for the true robustness of the data. Consequently, ERA is contingent upon the time of assessment and is specific to a particular attack or a defined set of attacks. In safety-critical domains especially, a stronger (or formal) guarantee of robustness is preferred. One solution to this cat-and-mouse game is to train a model capable of \emph{certifying} the \emph{local robustness} of its predictions within a small neighborhood in the input space---the main focus of this paper. 

Formally, given a classifier $F: X \rightarrow [m]$, an $\epsilon$-bounded local robustness certifier $\certifier: X \rightarrow \{0,1\}$ is a function that returns a booleen output to indicate if $F$ is locally robust at $x$. A \emph{sound} cetifier (Def.~\ref{def:cert}) is therefore free of false positive.

\begin{definition}[Soundness of Certification]
\label{def:cert}
For a classifier $F: X \rightarrow [m]$, a certifier $\certifier: X \rightarrow \{0,1\}$ is sound if $\forall x \in X$,
\begin{align*}
    \certifier(x) = 1 \implies F \text{ is $\epsilon$-locally robust at }x.
\end{align*}
\end{definition}

% such that $\certifier(x)=1$ if $F$ can be certified $\epsilon$-locally robust at $x$ and 0 otherwise.

% Def.~\ref{def:cert} formalizes the notion of a local robustness certifier. A certifier is \emph{sound} if whenever $\certifier(x)=1$ it implies that $F$ is $\epsilon$-locally robust at $x$. 
Hereafter, whenever we refer to a certifier, we assume it is sound unless stated otherwise.

% \begin{definition}[Local Robustness Certifier]
% \label{def:cert}
% Given a classifier $F: X \rightarrow [m]$, an $\epsilon$-bounded local robustness certifier $\certifier: X \rightarrow \{0,1\}$ is a function such that $\certifier(x)=1$ if $F$  can be certified $\epsilon$-locally robust at $x$ and 0 otherwise.
% \end{definition}

Certifiers can be classified as either \emph{complete} or \emph{incomplete} based on the outcome when $\certifier(x) = 0$. Specifically, a \emph{complete} certifier is one where an adversarial example is identified through its algorithmic design. Conversely, a certifier is deemed \emph{incomplete} if it concludes with $\certifier(x) = 0$ either due to the discovery of an adversarial example or because it is unable to prove local robustness before reaching a predetermined termination condition. Therefore, it is important to note that if $\certifier$ is incomplete, $\certifier(x) = 0$ does not imply that an adversarial example necessarily exists within the $\epsilon$-ball.
% $\|\delta\|_p\leq\epsilon$ such that $F(x)\neq F(x+\delta)$ and 1 otherwise. On the other hand, whenever an \emph{incomplete} certifier outputs 1, it is indeed guaranteed that no such adversarial perturbation exists but a denial of certification (i.e. output 0) does not imply that such a perturbation necessarily exists.

\begin{figure*}[t]
    \centering
    \begin{subfigure}[b]{0.3\textwidth}
        \includegraphics[width=0.975\textwidth]{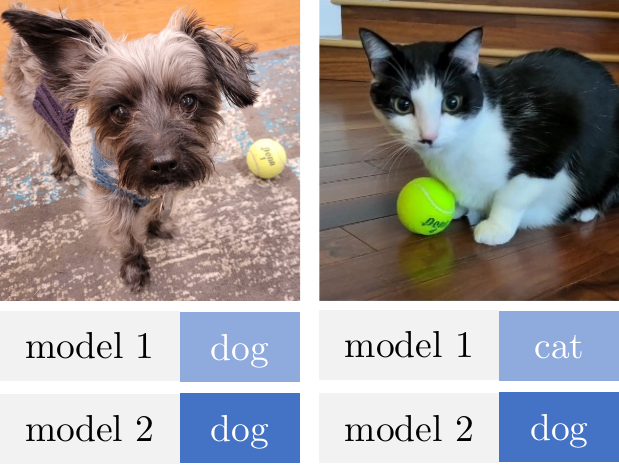}
        \caption{}
        \label{fig:cat_and_dog_1}
    \end{subfigure}
    \hfill
    \begin{subfigure}[b]{0.6\textwidth}
        \includegraphics[width=\textwidth]{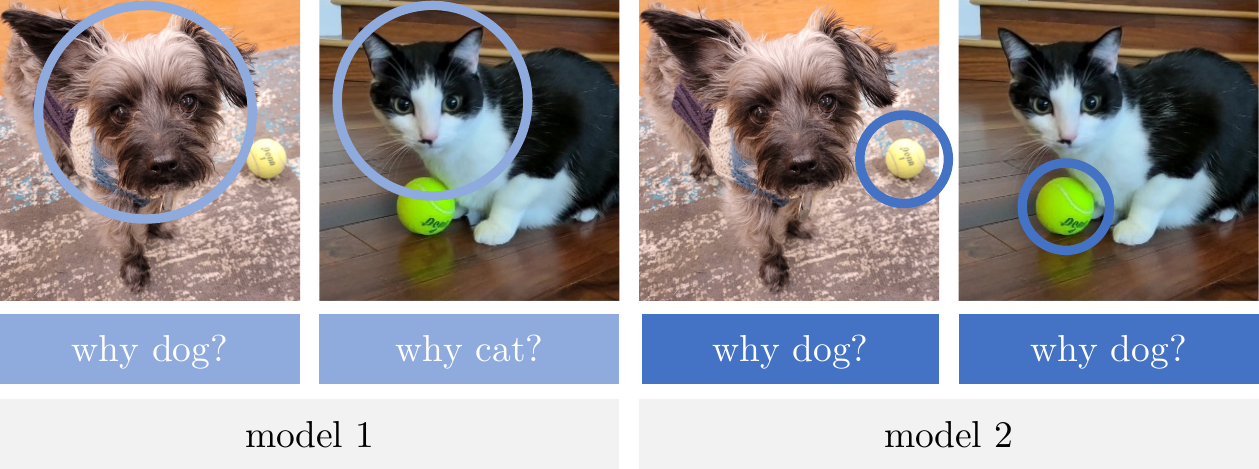}
        \caption{}
        \label{fig:cat_and_dog_2}
    \end{subfigure}
    \caption{An illustration of conceptual soundness. (a) Models make predictions on two inputs and model 2 is clearly wrong for the cat image. (b) Each model's internal logic for image predictions is summarized. Model 1 accurately uses the subject (i.e. the dog), whereas model 2 mistakenly focuses on the tennis ball, associating it with the dog class. We say in this case model 1 is conceptuall sound and model 2 is not.}
    \label{fig:conceptual_soundness}
\end{figure*}

Complete certification is an NP-complete problem for deep neural networks~\cite{katz2017reluplex}. Due to the high computational cost, existing complete certification methods~\cite{pulina2010abstraction, tjeng2017evaluating, katz2017reluplex, wang2018efficient, bunel2020branch} have difficulty in scaling to modern networks~\cite{li2023sok}.
Incomplete certifiers can be categorised into probabilistic and deterministic methods. Probabilistic methods provide local robustness guarantees with a probability\footnote{An acceptable probability is usually no lower than 99.9\% to participate in a public leaderboard (\url{https://sokcertifiedrobustness.github.io/leaderboard/}).}, meaning that there could exist false positive claims. Existing probabilistic methods~\cite{cohen2019certified,yang20rs_all_shapes,jeong2021smoothmix,carlini2022certified} add random noise to smooth classifiers and rely on Monte Carlo sampling to provide robustness bounds. As a result, these methods require extensive sampling during the certification process and are thus very expensive; for instance, they are typically evaluated on a 1\% subset of the ImageNet validation set for timing concerns~\cite{cohen2019certified, carlini2022certified, xiao2022densepure}. There are a great number of deterministic categories and here we include a few examples, e.g. internal bound propagation methods~\cite{Gowal2018OnTE, morawiecki2019fast}, linear relaxation methods~\cite{weng2018towards, salman2019convex, singh2019abstract, dvijotham2018dual}, solving a semi-definite programming~\cite{dathathri2020enabling, fazlyab2020safety, raghunathan2018certified}, and using the  global Lipschitz bounded networks~\cite{trockman21orthogonalizing, leino21gloro, araujo2023a}. A summary and comparison of certification method is shown in Fig.~\ref{fig:taxonomy}

% Among them, Lipschitz-bounded methods~\cite{hu2023scaling} achieve best results on small networks, for datasets such as CIFAR-10 and Tiny-ImageNet, and show the potential to scale up to ImageNet~\cite{hu2023scaling}. 
%In this paper, we main focus on this family of methods.

% \input{03_FoundationalWork}
% \input{04_WhyRobustness}
\section{\textbf{Why do we care about robustness research?}}
\label{sec:why_robustness}

In general, the concerns around robustness are uncontroversial, as it is widely recognized that adversarial examples pose a security vulnerability.
For completeness, we begin with an overview of the common justifications that ubiquitously motivate the robustness literature.
Although it is less universally acknowledged, in addition to its relation to security, robustness can be viewed more generally with respect to model quality, as a path to \emph{conceptual soundness}---the tendency of a model to use high-level features that are recognized by humans as appropriate for explaining its inferences.
We also contend with the objection that, despite the importance of robustness as a model property, \emph{robustness research} may be redundant from the perspective of \emph{scaling laws}, which perhaps suggest that data and model capacity are sufficient for fixing any model shortcomings.
Finally, we discuss some recent observations from the formal safety analysis of vision-based cyber-physical systems suggesting that improved robustness of perception models can translate into improved overall safety of the system in a formally provable sense.

\paragraph{\textbf{Lack of Robustness Is A Security Issue}} Clearly, adversarial examples negatively impact the reliability of neural networks that are vulnerable to them—and constitute a security concern in safety-critical machine learning systems—as they lead to unexpected erroneous behavior on seemingly benign inputs. Numerous examples, as listed by \citet{carlini2019evaluating}, have shown that it is possible to generate physical adversarial examples, e.g. eyeglasses~\cite{adversarial:ccs16, advml:arxiv17} and T-shirts~\cite{xu2020adversarial} that fool facial recognition systems, patches that make autonomous vehicles recognize stop signs as speed limit signs~\cite{8578273}, web content that causes an ad-blocker to consider an advertisement as neutral content~\cite{Tramr2018AdVersarialPA}, etc.

As AI systems get more powerful, their potential misuse can pose a significant threat, as they can be deliberately employed to instigate widespread damage~\cite{hendrycks2023overview}. Modern large models are trained to align with ethical behaviors~\cite{ziegler2019fine, hendrycks2021ethics, ouyang2022training} so they reject following harmful instructions, e.g., providing instructions on creating bio-weapons. However, recent research has shown that the current aligned models are de-facto not adversarially aligned~\cite{maus2023adversarial, carlini2023aligned, zou2023universal, wei2023jailbroken}. Automated methods are able to break the guardrails and instruct large models to generate harmful contents for the adversary's own benefit as is shown in Fig.~\ref{fig:adv_examples_llm}. 

\paragraph{\textbf{Robustness Is Necessary for Conceptual Soundness}}
Consider the thought experiment originally set up by Leino~\cite{leino22thesis} as illustrated in Fig.~\ref{fig:conceptual_soundness}. Two models in Fig.~\ref{fig:cat_and_dog_1} perform image classification, yet model 2 erroneously categorizes a cat image as a dog.
Suppose we are able to ``talk'' to the models to discover the most important input features they rely on to make their predictions.
An explanation like the circled features in Fig.~\ref{fig:cat_and_dog_2} would reveal that the error in model 2 arises because it detects a tennis ball and associates it with the dog class.
While tennis balls might be linked with dogs in the training data, they are not the deciding factor to any human for distinguishing between cats and dogs.
Put differently, a tennis ball is \emph{not} a sound concept for classifying cats and dogs. Thus, we will say that model 2 is not \emph{conceptually sound}. By contrast, model 1 uses appropriate features to form its internal logic, and
can therefore be considered \emph{conceptually sound}. This quality is essential for trustworthiness in any machine learning model, even if they perform well otherwise.

Upon closer inspection, it becomes clear that the existence of adversarial examples in a
network must constitute a violation of conceptual soundness. The adversarial perturbation, which is semantically meaningless by construction, is nonetheless causally relevant
to changing the model’s prediction, as witnessed by the anomalous prediction it induces.
Thus, these adversarial perturbations can be thought of as corresponding to some ill-conceived concept that the network erroneously encodes and employs.

% \paragraph{\textbf{Lack of Robustness Can Lead to Malicious Use of Models}}
% The potential misuse of modern powerful AI systems poses a significant threat, as they can be deliberately employed to instigate widespread damage~\cite{hendrycks2023overview}. The current large models are trained to align with ethical behaviors~\cite{ziegler2019fine, hendrycks2021ethics, ouyang2022training} so they reject following harmful instructions, e.g., providing instructions on creating bio-weapons. However, recent research has shown that the current aligned models are de-facto not adversarially aligned~\cite{maus2023adversarial, carlini2023aligned, zou2023universal, wei2023jailbroken}. Automated methods are able to break the guardrails and instruct large models to generate harmful contents for the adversary's own benefit as is shown in Fig.~\ref{fig:adv_examples_llm}. 

%\paragraph{\textbf{Improving Robustness Needs Research Attention}}
\paragraph{\textbf{Scaling Itself Does Not Resolve Lack of Robustness}}
Arguably, the so-called ``renaissance'' of machine learning has been driven first and foremost by the dramatic scaling of model capacity and training set sizes.
Supported conceptually by PAC theories, and empirically by the impressive and continuous progress made in the last decade, many remain bullish on the prospect that the power of scaling is far from reaching its limits.
It might seem logical to assume that the absence of adversarial robustness stems from an insufficiently large training set, which fails to allow the model to generalize to corner cases, including adversarial examples.
One might hypothesize that by simply enhancing the model's performance and generalization, the problem of adversarial susceptibility would dissipate as both models and dataset sizes scale up.

Counter arguments towards this over-optimistic opinion are two-fold. 
From a theoretical perspective, it is often hard to nail down what precisely is entailed by a training ``distribution,'' however, it more clear that adversarial examples are, almost by design, off-manifold points that can not be easily sampled from distributions of natural inputs. Therefore, it is unclear that scaling laws driving generalization have meaningful bearing on adversarial robustness.
On the empirical side, evidence from experiments on the current largest classification models (and generative models) indicate that adversarial examples still persist, especially when developers rely solely on the standard training techniques.

% some thoughts
% robustness is not just gone as we are able to make the model perform better than it used to be. 
% continuing normal ML research wont just fix the robustness problem magically. 

\paragraph{\textbf{Robustness Can Improve System-level Safety}} Recent works~\cite{calinescu2022discrete,puasuareanu2023closed} on proving safety of vision-based cyber-physical systems have observed that higher accuracy of the models used for vision-based perception can lead to stronger formal safety guarantees for the overall system. It has also been observed that, in practice, models trained in a robustness-aware manner show a correlation between accuracy and robustness~\cite{calinescu2022discrete}---the model tends to be more accurate on points where it is locally-robust. These two observations suggest that improving the robustness of the classifier used for perception can have a direct impact on the safety of the overall system.

% \begin{figure}
%     \centering
%     \includegraphics[width=0.45\textwidth]{images/relation_adv.pdf}
%     \caption{Adversarial robustness, a component of ML model security, constitutes an integral facet of the wider realm of AI safety and the ethics of machine learning research. The robustness of a model has been demonstrated to exhibit a significant correlation with various societal dimensions, such as privacy, trustworthiness, and fairness. \textcolor{red}{Zifan: Okay I was trying to draw a graph like this to show the value of robustness research but I just spoke to David Wagner and he thinks the right three circles are separated from robustness research and may actually cause confusions. Wdyt?}}
%     \label{fig:enter-label}
    
% \end{figure}
\section{\textbf{Why do we care about $\ell_p$ robustness?}}\label{sec:defense-for-lp-robustness}
\label{sec:why_lp}
% - additional benefits such as interpretability, calibration(?), useful primitive for non-$\ell_p$ robustness \\
% - in the most general case, semantic robustness cannot be formalized

Constraining adversarial perturbations within a small \( \ell_p \)-ball around a benign input is inspired by the spirit of perceptual similarity. Often, humans cannot detect these minuscule changes, so it is desirable for deep models to emulate this human trait. In other words, if the perturbations are noticeable to humans and they change the semantic meaning of the input, it wouldn't be reasonable to expect the model's outputs to remain unchanged. Besides this intuitive motivation, there are a number of other reasons in favor of researching $\ell_p$ robustness.

\paragraph{\textbf{$\ell_p$ Robustness Is the Bedrock for Non-$\ell_p$ Robustness}} 

An $\ell_p$-ball is not the only set that can represent semantic similarities. For example, in visual object detection, rotations and translations of objects often do not change what humans would label them with; however, the rotation set of an image cannot be represented by a \( \ell_p \)-ball. Similarly, in classifying text data, some gibberish strings do not directly fit into the \( \ell_p \)-ball of the input. However, methods developed for $\ell_p$ robustness also help to formulate the definitions in these non-$\ell_p$ setups. For example, while the set of rotations of an image is not an $\ell_p$-ball, research works attempt to parameterize the rotation matrix and bound the corresponding parameters in a $\ell_p$-ball~\cite{wang2022art, yang2023provable}, which also applies to other general matrix-defined transformations~\cite{dumont2018robustness, hao2022gsmooth, balunovic2019geometric, wu2023toward, hsiung2023caa}. Moreover, recent findings show that the improved $\ell_p$ robustness can lead to the improved robustness for non-$\ell_p$ robustness~\cite{mao2019metric, lin2020dual}.

\paragraph{\textbf{General Semantic Similarity Cannot Be Formally Specified and Certified}}

Specifying the (perceptual) similarity between inputs as $\ell_p$ distances is the cornerstone to develop formal verification tools towards ending the cat-and-mouse game between attackers and defenders.
Arguably, the general notion of semantic similarity between any two images cannot be specified in a formal way that can be used for certifying robustness. It is perhaps possible for image-based searching or contrast learning to use a second deep model to measure the similarity between pairs of inputs. In this case, the semantic similarity is not specified but determined with the opaque (and probably not reliable) inner workings of the judging model. In the worst case, the adversary can just fool the classifier and the similarity model at the same time and any certification relying on such judging model is always unsound (i.e. the certifier may return false negative results).

% Therefore, the consensus is that while an \( \ell_p \)-norm-bounded adversary is not the only and sufficient way to encode semantic similarity, it remains an essential and necessary foundation for preserving robustness in more advanced configurations.

\paragraph{\textbf{$\ell_p$ Robustness Connects to Other ML Fields}} Besides adversarial robustness, $\ell_p$ robustness can lead to many other desirable properties such as out-of-distribution robustness~\cite{NEURIPS2020_b90c4696, setlur2022adversarial}, improved transfer learning~\cite{deng2021adversarial, salman2020adversarially}, and increased interpretability and trustworthiness of the model's prediction~\cite{Wang2023, pmlr-v162-wang22e, Etmann2019OnTC, tsipras2018robustness} as well as help with generalization~\cite{xie2020adversarial}, understanding memorization of training points~\cite{leino22privacy}, learn invertible functions~\cite{behrmann2019invertible} and build stable GANs~\cite{zhong2020improving}. 

% interpretability 
\section{\textbf{Why do we care about certification?}}\label{sec:y-certification}

\begin{figure*}
    \centering
    \includegraphics[width=0.9\textwidth]{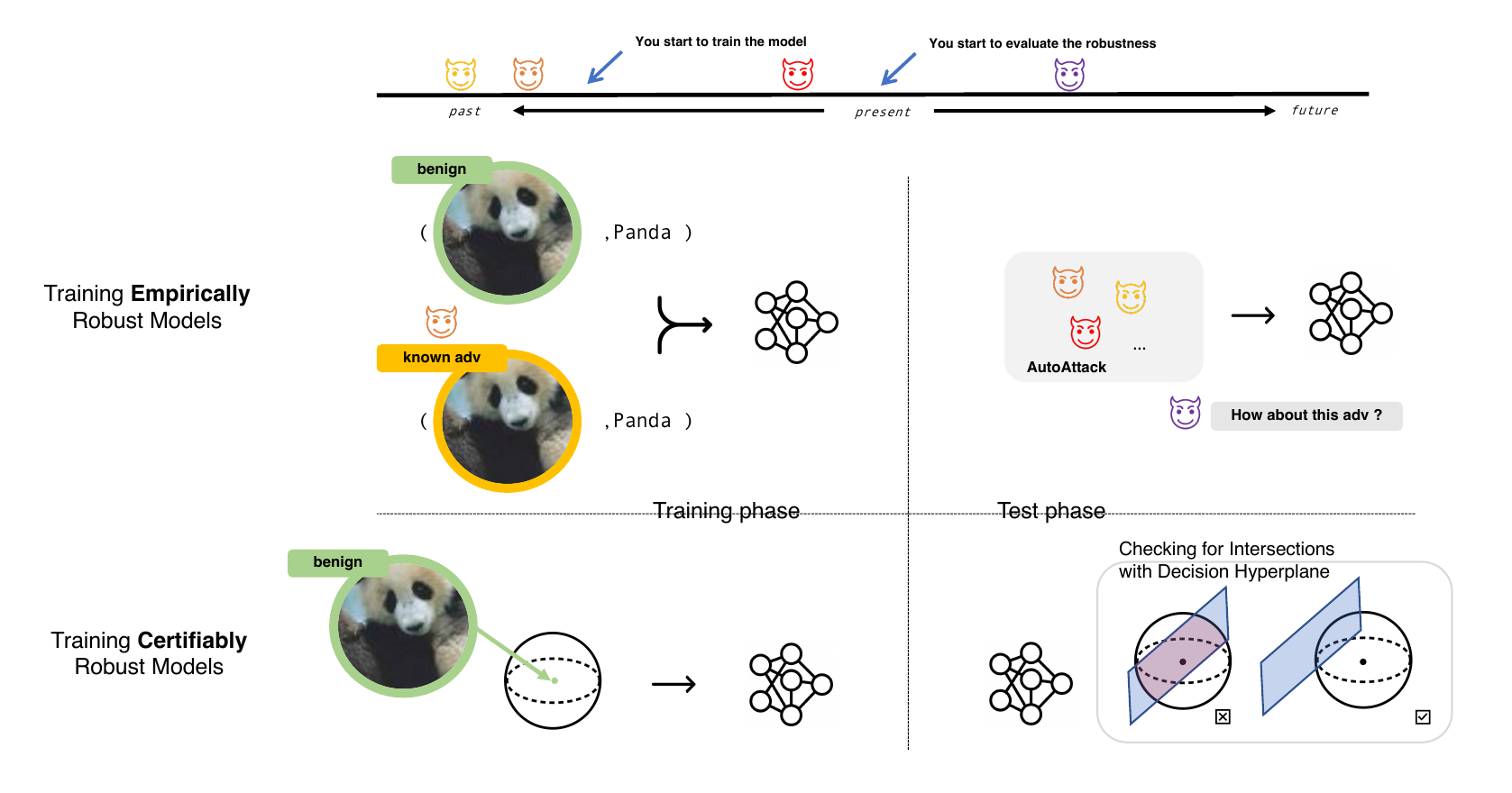}
    \caption{An illustration of the differences between training empirical robust models and certifiably robust models. AutoAttack~\cite{croce2020reliable} is a set of attacks commonly used to report the empirical robust accuracy. Of particularly note, the process of the certifying robustness is adversary-free, therefore, the model is guaranteed to be robust at certified points with future (i.e. unknown) attacks.  }
    \label{fig:era-vs-vra}
\end{figure*}

%\klas{should have some introductory paragraph here.}

Given the arguments in previous sections demonstrating that robustness research is important and $\ell_p$ norm-bounded adversaries represent a theoretically and practically useful threat model,  certification seems desirable to guarantee protection against such adversaries. However, we clarify in Section~\ref{sec:cert:notions} that the notion of certification depends on the context and means different things at different stages of the learning pipeline. Nevertheless, at each of these stages, certification can yield formal guarantees that can help escape from the cat-and-mouse game between attackers and defenders. The nature of these guarantees is made precise in Section~\ref{sec:cert:catandmouse}. Furthermore, in theory, there is no fundamental trade-off between accuracy, robustness, and efficient certifiability---under mild distributional assumptions, there always exists a Lipschitz-bounded function that is perfectly accurate and robust (Section~\ref{sec:cert:training}). We therefore argue that certified training that incorporates Lipchitz-based certification offers a promising path towards learning this theoretically feasible and desirable function.

\subsection{Notions of Certification}
\label{sec:cert:notions}
% \ravi{I think we should clarify that certification can appear at three different stages: during training, post-training, and post-deployment, i.e., at inference-time. While we do not have formal arguments for why using certification during training is desirable, there are informal reasons to prefer certified training over adversarial training and vanilla, non-robust training (further elaborated in section 5C). On the other hand, there are clear, formal reasons to prefer certified over empirical defenses, both, in the post-training and post-deployment stages (elaborated in section 5B). Then we can make the point that efforts like the Neural Network Verification competition that just focus on post-training certification are overly narrow.}

Prior to making an argument supporting the utility of robustness certification, it is worth distinguishing between three particular settings in which certification may be applied, which correspond to three stages of the learning pipeline: training, evaluation (post-training, pre-deployment), and inference (post-deployment).
Certification methods may apply to more than one stage, but not all methods are practically applicable to every stage.
Unsurprisingly, the value of certification depends on the context(s) it can be applied to.
In the remainder of Section~\ref{sec:cert:notions}, we will provide a brief overview of the important distinctions between \emph{certified training}, \emph{post-training certification}, and \emph{inference-time certification}, allowing us to frame our discussion of utility accordingly.

%\subsection{Post-hoc Certification}
%\label{sec:cert:posthoc}

\paragraph{\textbf{Certified Training}}
Certified training refers to any training method that directly or indirectly incorporates robustness certification into its training objective.
Typically, this involves incorporating a sound certification procedure into the model's loss function, however many approaches simply employ heuristics that make certification with some sound procedure easier.
Certified training is generally used in conjunction with a certification procedure that is applied post-training and/or at inference time; however, within this paradigm, the certification procedure may assume control over the structure and parameters of the model.
Generally speaking, not all certification procedures are amenable to certified training, as training requires differentiability, and sufficient performance and parallelizability to be practically run over thousands of batches throughout training.

\paragraph{\textbf{Post-training Certification}}
Post training certification applies a certification procedure to a validation dataset in order to estimate the Verified Robust Accuracy (VRA), i.e., the fraction of points on which the model is both correct \emph{and} certifiably robust.
Post-training certification can be applied either post hoc (i.e., assuming we have no control over the model that must be certified), or in conjunction with certified training.
In the latter case, incomplete certifiers that would otherwise provide %hopelessly 
loose robustness guarantees may (and indeed, do) become highly practical, particularly because of their superior speed~\cite{gowal2018effectiveness,leino21gloro}.
Moreover, while post-training certification does not necessitate ``real-time'' performance, complete methods---exponential by nature---may be intractable, harming the measured VRA whenever the procedure fails to terminate. 

\paragraph{\textbf{Inference-time Certification}}
Inference-time certification provides additional security for an ML-based system during deployment by using a certification procedure to produces a point-wise robustness certificate to accompany each prediction.
The certificate can be handled by the system in various ways, e.g., by rejecting inputs that cannot be certified, or by flagging them for human involvement.
% The key distinction between post-training and inference-time certification ...
Like post-training certification, inference-time certification can be applied post hoc, or in conjunction with certified training.
Typically, methods capable of post-training certification can also be applied at inference time; however, in deployment settings, performance (in terms of speed) is often indispensable, thus ``real-time'' methods are typically required.

% Separate paragraph from rest of section.
\vspace{1em}

We contend that \emph{certification is primarily useful within the certified-training paradigm}.
While post hoc certification is an interesting research problem, the inability to apply tailored regularization to the model being certified kneecaps our ability to leverage highly efficient (but incomplete and potentially loose) certification procedures, which have the speed to make inference-time certification practical. 
Meanwhile, complete certification is known to be NP-complete on general networks~\cite{katz2017reluplex}, 
% offering little hope that post hoc methods will scale to the high-capacity networks that would be required\footnote{See \cite{bubeck2021law,leino23capacity} for a discussion of the capacity requirements for (certifiably) robust networks.} for state-of-the-art certified performance, 
severely limiting the application of complete post hoc methods in practice,
especially at inference time.
For a sense of scale, 
% the VNN-COMP neural network verification competition showcases state-of-the-art post hoc certification methods and their capabilities~\cite{bak2021vnn-comp};
state-of-the-art post hoc certification techniques can require as much as \emph{five minutes} per instance to be able to terminate on up to 80\% of instances on a small, adversarially trained ResNet architecture~\cite{bak2021vnn-comp}.
Moreover, where comparable, the VRAs reported for top certified training procedures substantially surpass those recorded for post hoc certification of adversarially trained networks.\footnote{See \cite{li2023sok} for an overview of state-of-the-art VRA for various datasets, $\ell_p$ norms, and perturbation budgets. These can be compared to VRA numbers reported in the VNN-COMP neural network verification competition~\cite{bak2021vnn-comp}. Notably, VNN-COMP only evaluates against $\ell_\infty$ perturbations, which can be analyzed by IBP techniques; no post hoc certification method has recorded nontrivial results for $\ell_2$ perturbations on adversarially trained models of the scale presented in VNN-COMP.}

On the other hand, as we will discuss in Section \ref{sec:cert:training}, tailored certified training techniques are a promising avenue for achieving \emph{higher} VRA with real-time certification procedures.
Overall, these considerations suggest that efforts that focus only on post-training robustness certification are overly narrow.
% \klas{think we should also talk about linf restrictions and l2 limitations here. There is no way that any of the methods from the competition would terminate on a non-trivial fraction of instances in the l2 case.}

% Post-hoc certification ala Neural Network Verification competition doesnt make sense. For instance, if a certification method cannot be used during training, is it useful? 

\subsection{Empirical vs Certified Defenses: Escaping the cat-and-mouse game}
\label{sec:cert:catandmouse}
%Why do we care about certified vs empirical defenses? i.e. assuming we can certify, why is this useful? \\
%-Avoid cat-and-mouse game: Make this formal

In this section, we discuss the reasons for preferring certified defenses over empirical ones in the post-training and post-deployment (i.e., inference-time) phases of the model lifecycle. An oft-cited reason in the robustness literature in the favor of certified defenses is that they help escape the cat-and-mouse game between norm-bounded adversaries and defenders--- certification ensures that no norm-bounded adversary can successfully attack the model. Indeed, if a model is certified $\epsilon$-locally robust at a point $x$ (with respect to some $\ell_p$ norm), no $\epsilon$-bounded perturbation of $x$ can affect the model output and the model is protected from any norm-bounded adversary at $x$. This, however, is a \emph{local} guarantee. We find that the nature of the \emph{global} guarantee granted by certification is often left unspecified in the literature. In addition, as discussed in Section~\ref{sec:cert:notions}, the notion of certification itself can mean different things depending on the context. We clarify and formally express the guarantees realized via certification here.

%\subsubsection{Post-training Certification}
\paragraph{\textbf{Post-training Certification}}

The goal of post-training certification is to \emph{measure} the percentage of points in a held-out dataset where the model is both accurate and certified as $\epsilon$-locally robust. This percentage, usually referred to as the empirical Verified Robust Accuracy ($\empvra$), can help estimate the model accuracy on unseen data in the presence of any norm-bounded adversary. $\empvra$ on a dataset $d$ with $n$ labeled $(x_i,y_i)$ pairs is defined as
$$\empvra(F,d):=\frac{1}{n}\sum_{i=1}^{n} \mathbbm{1}_{F(x_i)=y_i \wedge \certifier(x)=1}.$$
where $\certifier$ is a complete or incomplete certifier.

The corresponding definition of Verified Robust Accuracy ($\vra$), in expectation, with respect to a distribution $\mathcal{D}$ is
$$
\vra(F,\mathcal{D}):=\underset{x\sim \mathcal{D}}{\mathbb{E}}[\mathbbm{1}_{F(x)=F^*(x) \wedge \certifier(x)=1}]
$$
Theorem~\ref{thm:vra} formally states the relationship between $\vra$ of the model on the non-perturbed distribution $\mathcal{D}$ and its accuracy ($\acc$\footnote{$\acc(F,\mathcal{D}):=\underset{x\sim \mathcal{D}}{\mathbb{E}}[\mathbbm{1}_{F(x)=F^*(x)}]$}) with respect to all possible perturbations of $\mathcal{D}$ (see Def.~\ref{def:dist_per}) that can be the result of a norm-bounded adversary.

\begin{definition}[Perturbations of a Distribution]
\label{def:dist_per}
Let $\mathcal{D}$ be a distribution over $X$, $B_\epsilon:=\{v~|~||v||_p \leq \epsilon\}$ be the ball of norm $\epsilon$-bounded vectors, and $\Delta_\epsilon:=\{\Delta~|~\support(\Delta)\subseteq B_\epsilon\}$ be the set of all distributions whose support is in $B_\epsilon$.
Then the set $\perturb({\mathcal{D}})$ of perturbed distributions with respect to $\mathcal{D}$ that are the result of any norm-bounded adversary is defined as:
\begin{align*}
    \perturb(\mathcal{D}) := \{x \sim \mathcal{D}; \delta \sim \adversary(x);& \text{return } (x+\delta)~|\\
    &\adversary \in X \rightarrow \Delta_\epsilon\}
\end{align*}
$\adversary$ represents a norm-bounded adversary that maps each input to a perturbation distribution (for a deterministic adversary, this is a Dirac delta distribution) from $\Delta_\epsilon$.
\end{definition}

Theorem~\ref{thm:vra} uses the notion of a distribution $\mathcal{D}$ being $\epsilon$-separable with respect to a classifier $F:X \rightarrow [m]$. Intuitively, if any pair of points in the support of $\mathcal{D}$ (denoted as $\support(\mathcal{D})$) are assigned different labels by $F$, then, for $\mathcal{D}$ to be $\epsilon$-separable, the points should be at least $2\epsilon$ apart. Def.~\ref{def:sep} states this formally.

\begin{definition}[$\epsilon$-Separability]
\label{def:sep}
Let $F:X \rightarrow [m]$ be a classifier and $\mathcal{D}$ be a distribution over $X$. Then $\mathcal{D}$ is $\epsilon$-separable with respect to $F$ if,
$$
\forall x,x' \in \support(\mathcal{D}).~F(x)\neq F(x') \implies ||x-x'||_p \geq 2\epsilon
$$
\end{definition}

\begin{theorem}[Post-training Certification]
\label{thm:vra}
Let $F:X \rightarrow [m]$ be a classifier, $F^*:X \rightarrow [m]$ be the corresponding ground-truth labeling function, and $\mathcal{D}$ be an $\epsilon$-separable distribution over $X$ wrt $F^*$ from which unperturbed inputs for the classification task are drawn. 
If $
\forall x \in \support(\mathcal{D}), F^*(x) \text{ is } \epsilon\text{-locally robust}$, then
$$
\forall \mathcal{D}' \in \perturb(\mathcal{D}).~\vra(F,\mathcal{D})\leq \acc(F,\mathcal{D}') 
$$

\end{theorem}
\begin{proof}
% (Sketch) Consider a perturbed distribution $\mathcal{D'} \in \perturb(\mathcal{D})$. For every point $x$ in the support of $\mathcal{D}$, there is a set of points $X'$ %%(containing a single point if $\adversary$ is deterministic) 
%such that $\forall x' \in X'. ||x-x'||_p \leq  \epsilon$ and the cumulative measure of $X'$ as per $\mathcal{D'}$ is no less than the measure of $x$ as per $\mathcal{D}$ (from the definition of $\mathcal{D'}$). Also, for every point $x$ such that $\mathbbm{1}_{F(x)=F^*(x) \wedge \certifier(x)=1}$, it holds that for all $x'$ such that $||x'-x||_p \leq \epsilon$, $F(x')=F^*(x')$(from the definition of $\certifier$, the \emph{separability} condition, and the robustness of $F^*$). From these facts, we get the desired conclusion.
%-----------------
%Non-measure theoretic proof
%Let's arbitrarily draw $x \sim \mathcal{D}$ and $\delta \sim \Delta_\epsilon(x)$ where $\forall \delta \in \Delta_\epsilon(x), ||\delta||_p \leq \epsilon$ so $ \Delta_\epsilon(x)$ denotes all possible norm-bounded perturbations at $x$. 
Let $\mathcal{D'}$ be an arbitrary distribution from $\perturb(\mathcal{D})$. 
\begin{align*}
&\acc(F,\mathcal{D'})\\
&=\underset{x\sim \mathcal{D'}}{\mathbb{E}}[\mathbbm{1}_{F(x)=F^*(x)}]\\
&=\underset{x\sim \mathcal{D}, \delta \sim \adversary(x)}{\mathbb{E}}[\mathbbm{1}_{F(x+\delta)=F^*(x+\delta)}]~(\text{Definition of } \mathcal{D'})\\
&=\underset{x\sim \mathcal{D}, \delta \sim \adversary(x)}{\mathbb{E}}[\mathbbm{1}_{F(x+\delta)=F^*(x)}]~(F^*\text{ is locally robust})\\
&\geq\underset{x\sim \mathcal{D}}{\mathbb{E}} [\underset{\delta \in B_\epsilon}{\min}[\mathbbm{1}_{F(x+\delta)=F^*(x)}]]~(\text{Worst-case perturbation})\\
&\geq\underset{x\sim \mathcal{D}}{\mathbb{E}} [\mathbbm{1}_{F(x)=F^*(x) \wedge \certifier(x)=1}]\\
&=\vra(F,\mathcal{D})
\end{align*}
\end{proof}

Theorem~\ref{thm:vra} states that the $\vra$ of the model gives a lower bound on the expected accuracy in the presence of a norm-bounded adversary. No matter how the adversary perturbs the samples, the accuracy of the model can never be worse than the $\vra$ on the original distribution. This result holds under a mild assumption about the data distribution $\mathcal{D}$ and the ground-truth labeling function $F^*$. The assumption states that every pair of points in the support of $\mathcal{D}$ with different ground-truth labels are separated by at least $2\epsilon$ and furthermore, the ground-truth labeling function is $\epsilon$-locally robust at every such point in the support of $\mathcal{D}$. The $\epsilon$-separability assumption is empirically well-motivated~\cite{yang2020closer}; in fact, if this assumption is not true, then there is always an inherent trade-off between robustness and accuracy of the model. The second part of the assumption is a basic motivation for aiming to learn robust classifiers and is feasible due to the $\epsilon$-separability of $\mathcal{D}$.

In contrast, when using empirical defenses, one measures the percentage of points in the held-out dataset where the model is accurate and not susceptible to misclassification subjected to a \emph{fixed} set of attacks, often referred to as the Empirical Robust Accuracy of the model ($\empera$). However, $\empera$ and its corresponding version in expectation, i.e., $\era$, can only guarantee model accuracy with respect to the fixed set of attacks considered but not with respect to any norm-bounded adversary. 

Note that the theorem relates $\vra$ with the accuracy of the model on perturbed samples but we are only able to measure $\empvra$. However, $\empvra$ and $\vra$ are related via standard statistical results such as Hoeffding's Inequality. 

%\subsubsection{Inference-time Certification}
\paragraph{\textbf{Inference-time Certification}}

A model with inference-time certification takes the following form,
\begin{equation*}
F^{\bot}(x)~:=~\text{if } (\certifier(x)=1) \text{ then }F(x) \text{ else } \bot
\end{equation*}
where $F$ is the model and $\certifier$ is the local robustness certification procedure. Consider the case where $\certifier$ returns 1, i.e., the model is certified $\epsilon$-locally robust at $x$. In this case, even if $x$ is an adversarially perturbed input, we are guaranteed that the model prediction at $x$ is the same as the original unperturbed input. In this sense, no matter how the norm-bounded adversary generated the perturbed input $x$, the certification guarantees that the model is unaffected by the adversarial intervention whenever the check passes\footnote{A failure of the check does not come with any guarantee, and can, in fact, be exploited by an adversary to degrade model utility~\cite{leino2022degradation}}. In contrast, an empirical check of $\epsilon$-local robustness or any other inference-time empirical defense based on ``purifying'' or modifying the input before the model accesses it~\cite{croce2022evaluating,frosio2023best} cannot guarantee that the model prediction is unaffected by an adversary.
Theorem~\ref{thm:runtime} formally expresses this guarantee.

\begin{theorem}[Inference-time Certification]
\label{thm:runtime}
Given a model $F^{\bot}$ with inference-time certification, the following holds $\forall x\in X$, 
$$
 F^\bot(x) \neq \bot \implies x \text{ is not an adversarial example}.
$$
\end{theorem}
\begin{proof}
Assume $F^\bot(x) \neq \bot$. Then, by definition of $F^\bot$ and $\certifier$, $\forall x' \in X. ||x-x'||_p \leq \epsilon \implies F(x)=F(x')$. Since our adversary is norm-bounded, the original unperturbed input $x''$ has to satisfy the property that $||x''-x||_p \leq \epsilon$. Therefore, it follows that $F(x)=F(x'')$, i.e., $x$ is classified in the same way as the original input and is therefore not an adversarial example. 
\end{proof}

Note that inference-time certification requires the certifier to be highly efficient. While techniques based on smoothing and sampling, constraint solving, linear relaxations, and semi-definite programming are too expensive to be used at inference-time, Lipschitz-based certifiers are ideally suited for this purpose since the global Lipschitz constant of the model can be calculated offline.

\subsection{Accuracy, Robustness, and Certification}
\label{sec:cert:training}

While robustness certification can resolve the cat-and-mouse game played between norm-bounded attackers and defenders, one may contend that it is unclear at what cost this comes.
Statements made consistently throughout the literature over the years~~\cite{fawzi2018analysis,tsipras2018robustness,zhang2019theoretically} suggest that a common belief is that accuracy, robustness, and sound certification are in conflict with one another, such that an inherent trade-off is induced.
However, in this section we present an alternate view.
First, under reasonable distributional assumptions, there is no conflict between accuracy and local robustness.
Moreover, a recent result~\cite{leino23capacity} shows that within the certified training paradigm, we can hope to steer the model towards one that is not only accurate and robust, but also efficient to certify (via a Lipschitz check), provided our hypothesis class has sufficient capacity~\cite{leino23capacity}.
Finally, in light of these observations, we argue that Lipschitz-based certified training may be the most effective path forward for learning high-performance robust models, and thus merits further research attention.

%\subsubsection{Accuracy and Robustness}
\paragraph{\textbf{Accuracy and Robustness}}

Much of the literature references, directly or implicitly, a robustness-accuracy trade-off~\cite{fawzi2018analysis,tsipras2018robustness,zhang2019theoretically}.
\citet{tsipras2018robustness} formalize this notion, showing a theoretical trade-off between ``standard'' and robust accuracy in a simple Gaussian setting.
However, it is not clear that the distributional assumptions of Tsipras et al.~\cite{tsipras2018robustness} are reasonable in any situation where we would hope for robustness.
Notably, under these distributional assumptions, the data are not separable.
This seems clearly problematic, as two in-distribution (natural) points may be arbitrarily close to one another (even identical), but have different ground-truth labels; while implicit in the notion of robustness---even the most general notion of ``semantic'' robustness---is the assumption that the ground truth is robust.
After all, adversarial examples are meant to cause misclassifications, not arbitrary alternations between equally valid labels.

\citet{yang2020closer} have argued that $\epsilon$-separability (Definition~\ref{def:sep}) is a more reasonable assumption in the context of robustness.
Crucially, $\epsilon$-separability ensures that adversarial examples with respect to the ground truth are not within the support of the distribution.
This is a natural requirement, which holds empirically on standard datasets~\cite{yang2020closer}; and it is clear that under this distributional assumption, no accuracy-robustness trade-off exists.

%\subsubsection{On the Incompleteness of Efficient Certification}
\paragraph{\textbf{On the Incompleteness of Efficient Certification}}

State-of-the-art (in terms of maximum $\empvra$ achieved) certification procedures are \emph{incomplete}, using certified training techniques to amplify the effectiveness of efficient robustness checks.
As discussed in Section~\ref{sec:cert:notions}, the fact that complete certification is NP-complete leaves little hope that complete procedures will ever be practical on the large-scale neural networks used in many real-world applications.
This raises the natural question: \emph{at some point, won't the incompleteness of leading approaches compromise the model's performance?}
If so, one may have to contend with whether or not the strong guarantees robustness provides are worthwhile.

Note that there are two concerns here. The first is that a model may be robust, but not all points will be certifiable, causing us to underestimate its robust performance, and falsely flag robust points at inference time.
The second is that when using an incomplete certification as part of a certified training procedure, the learned model will become over-regularized damaging its true robustness (certified or not) and accuracy.

While both of these are valid concerns, when we consider the context of certified training, the picture becomes less bleak. 
% \citet{leino23capacity} has argued that 
When the certifier has control over the model's parameters during training, it is pertinent to ask not whether \emph{all} functions within the hypothesis class can be tightly certified (corresponding to completeness), but rather, whether it is possible to learn a function that (1) is robust on the distribution of interest, and (2) can be tightly certified.
Perhaps surprisingly, a result of \citet{leino23capacity} shows that for any $\epsilon$-separated data distribution, there always exists a function that achieves both of these goals, using a simple Lipschitz check to \emph{tightly} decide the robustness all points.
This function has the property that the exact distance to the decision boundary (i.e., the robustness radius) is always equal to $\nicefrac{\Delta}{K}$, where $\Delta$ is the minimum margin between the logit value of the predicted class and any other logit value, and $K$ is the function's global Lipschitz constant, which is (on this function) everywhere the same as its local Lipschitz constant. This means it can be trivially certified at inference time using a bound on the global Lipschitz constant.

The existence of such an ``ideally robust'' function does not directly suggest that it is easy, or even practical, to learn such a function, however.
In fact, \citeauthor{leino23capacity} points to the fact that approximating this ideally robust function using the standard piece-wise linear hypothesis class requires excess capacity beyond what would be required to learn a boundary that is technically robust (but not certifiable)---and this is in addition to the practical challenges around training dynamics and Lipschitz estimation.
Nonetheless, we believe that the existence of the ideally robust function should inspire hope that Lipschitz-based certified training has a great deal of potential to minimize the trade-off between accuracy, robustness, and efficient certification.
Indeed, even now, certified training approaches based on global Lipschitz-based certification have emerged thus far as the clear leaders for state-of-the-art $\empvra$~\cite{hu2023recipe}.

% \citet{leino23capacity} has argued that when the certifier has control over the model's parameters during training, a more relevant concept of completeness is that of \emph{training-complete on the hypothesis class} (Definition~\ref{def:complete_hypothesis}).
% Essentially, a certification procedure, $C$, is training-complete on a hypothesis class, $\mathcal F$, if for all decision boundaries achievable by functions from $\mathcal F$, there exists some function from $\mathcal F$ implementing that boundary on which tight certification is possible with $C$.

% \begin{definition}[Training Complete on Hypothesis Class]
% \label{def:complete_hypothesis}
%     For a model, $F: X \to [m]$ where $F(x) = \text{argmax}_i{f(x)_i}$, let $R(F, \epsilon)$ be the set of $\epsilon$-robust points, i.e.,
%     $$
%     R(F, \epsilon) = \{x \in X : \text{$F$ is $\epsilon$-locally robust at $x$}\}.
%     $$

%     \noindent
%     Let $C$ be a sound certifier.
% \end{definition}

%\subsubsection{The Power of Lipschitz-controlled Training}
\paragraph{\textbf{The Power of Lipschitz-controlled Training}}

As discussed, the limitations which may lead to a trade-off between accuracy, robustness, and certification are not as fundamental as one might initially believe.
Let us discuss this point further, as it relates to the motivation for studying robustness certification.

Specifically, we would like to draw more attention to the recent success and progress of Lipschitz-based certification.
Within just the last few years, the state-of-the-art deterministic VRA, led by Lipschitz based methods, e.g., \cite{leino21gloro,trockman21orthogonalizing,hu2023scaling,hu2023recipe,soc}, has increased by a factor of over 50\% on benchmark datasets like CIFAR-10; and certification has become possible on large datasets (like ImageNet) and architectures (like deep ResNets)~\cite{hu2023scaling}.
Figure~\ref{fig:acc_vs_vra} illustrates this progress, specifically in terms of closing the gap between accuracy and verified robust accuracy.
There is a strong empirical and theoretical basis for believing that these strides should hope to continue; we therefore contend that it is worth exploring further in this direction to build on the tremendous progress.

\begin{figure}[t]
    \centering
    \includegraphics[width=0.40\textwidth]{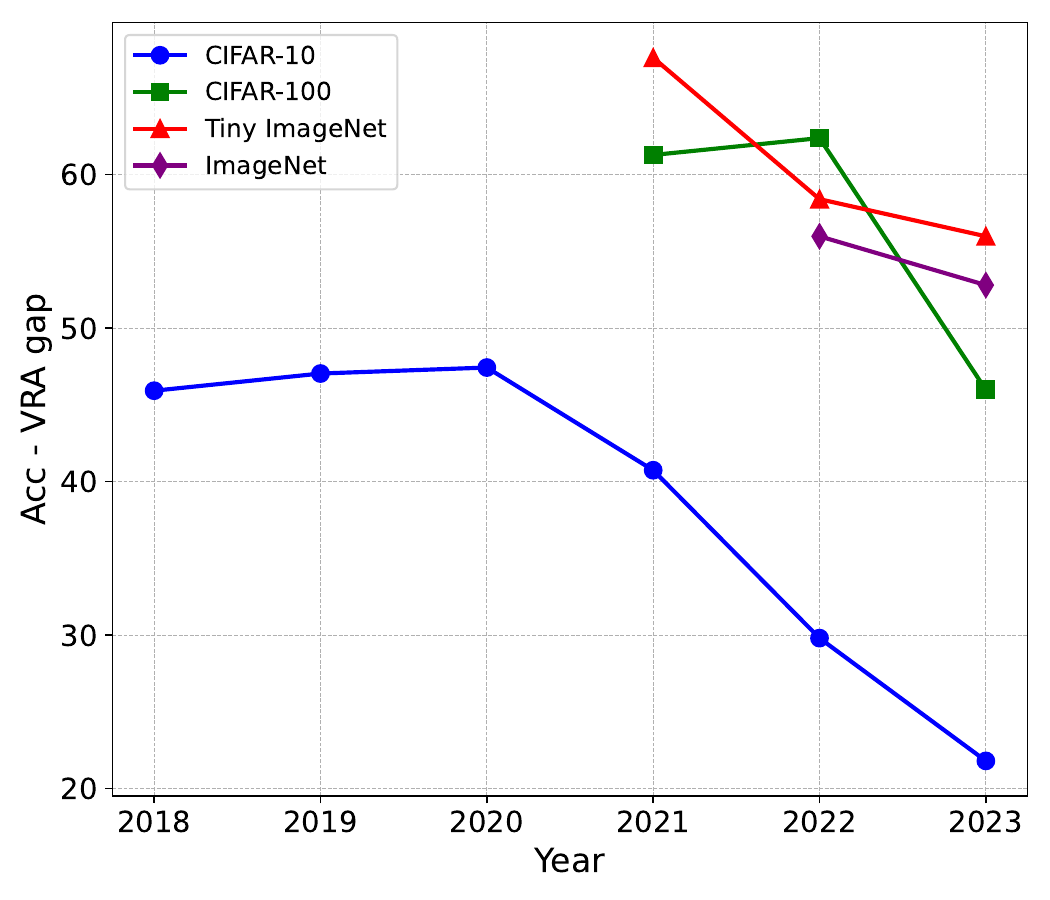}
    \caption{The gap between the reported SOTA $\empacc$ and SOTA $\empvra$ (measured locally in a $\ell_2$-ball with $\epsilon=36/255$) on several vision datasets.}
    \label{fig:acc_vs_vra}
\end{figure}

%\subsubsection{Improving Model Capacity for Robustness via Lipschitz-based Methods}
\paragraph{\textbf{A Way Forward for Robustness via Lipschitz-based Methods}}

Although the power of Lipschitz-based certification is supported theoretically, and has thus far been borne out empirically, our bullishness on this approach's future prospects deserves closer examination.
A series of studies find, either empirically or theoretically, that robust models require more  capacity~\cite{schmidt2018adversarially, bubeck2021law, li2023clean, wang2023better, altstidl2023raising}, particularly in the context of certification.
However, for certifiable models, especially those certified using Lipschitz-based methods, increasing the model capacity is not trivial.

To obtain satisfactory certification performance using Lipschitz-based training methods, it is crucial to design network architectures with easy-to-compute and tight Lipschitz bounds.
This is a challenge in and of itself, but it also specifically complicates capacity-scaling.

One common approach to adding capacity is by adding depth to the network, but this is challenging for Lipschitz-based methods, as the bound for the entire network is typically computed as a product of the layer-wise Lipschitz bounds, meaning that any looseness compounds multiplicatively with each layer.
Furthermore, deep models typically use residual architectures, but a naive bound on the Lipschitz constant of residual blocks is unfortunately loose~\cite{hu2023scaling}.
\citet{hu2023scaling} proposed using a linear residual branch so that the Lipschitz bound for the residual block can be tight, showing that this method can obtain non-trivial deterministic certification on ImageNet. However, linear residual blocks are less expressive than non-linear residual blocks.

The requirement for Lipschitzness, also renders traditional attention-based architectures as out-of-scope since they do not have a Lipschitz bound.
Although some work~\cite{kim2021lipschitz, xucertifiably} proposes to modify the attention mechanism to make it Lipschitz, so far attention-based mechanisms have not been shown to outperform basic ResNet architectures~\cite{hu2023scaling}.
Nonetheless, exploring transformer-like building blocks with tight Lipschitz bounds remains an interesting research direction.
Models using such blocks may benefit from the high capacity from Transformer-like deign, and bridge the gap between $\vra$ and standard accuracy ($\acc$) for visual models.
Additionally, ``Lipschitz transformers'' could open the door for providing certain security guarantees for LLMs, which have recently gained massive attention from the machine learning community as well as the general public. 

\section{\textbf{Conclusion}}
Though a tremendous amount of research has been produced over the last decade on adversarial robustness with respect to norm-bounded adversaries, the impact of this research in practice remains unclear.
% Is there a future where robustness certification will be comercially adopted?
In light of this situation, we revisit some basic underlying motivations driving this research and argue for theo position that robustness research with respect to norm-bounded adversaries, particularly research on certified defenses, continues to be practically worthwhile and technically challenging. We clarify that certification can mean different things at different stages of learning pipeline and formalize the guarantees granted by certification at these different stages. With this, we hope to bring clarity and structure to the ongoing, almost decade-long, conversation on certified robustness, and provide insight as to which directions remain fruitful. We also clarify that, under realistic distributional assumptions, there is no inherent trade-off between accuracy, robustness, and efficient certifiability, suggesting that certified models stand a hopeful chance of being able to satisfy the many requirements that would be necessary for adoption into real-world applications.
Finally, to this end, we propose specifically that certified training incorporating Lipschitz-based certification may offer the most promising path towards learning an ideal function that is accurate, robust, and certifiable.

\nocite{langley00}

\bibliography{references}
\bibliographystyle{icml2023}

%%%%%%%%%%%%%%%%%%%%%%%%%%%%%%%%%%%%%%%%%%%%%%%%%%%%%%%%%%%%%%%%%%%%%%%%%%%%%%%
%%%%%%%%%%%%%%%%%%%%%%%%%%%%%%%%%%%%%%%%%%%%%%%%%%%%%%%%%%%%%%%%%%%%%%%%%%%%%%%
% APPENDIX
%%%%%%%%%%%%%%%%%%%%%%%%%%%%%%%%%%%%%%%%%%%%%%%%%%%%%%%%%%%%%%%%%%%%%%%%%%%%%%%
%%%%%%%%%%%%%%%%%%%%%%%%%%%%%%%%%%%%%%%%%%%%%%%%%%%%%%%%%%%%%%%%%%%%%%%%%%%%%%%
\newpage
\appendix


% You can have as much text here as you want. The main body must be at most $8$ pages long.
% For the final version, one more page can be added.
% If you want, you can use an appendix like this one, even using the one-column format.
%%%%%%%%%%%%%%%%%%%%%%%%%%%%%%%%%%%%%%%%%%%%%%%%%%%%%%%%%%%%%%%%%%%%%%%%%%%%%%%
%%%%%%%%%%%%%%%%%%%%%%%%%%%%%%%%%%%%%%%%%%%%%%%%%%%%%%%%%%%%%%%%%%%%%%%%%%%%%%%

\end{document}